\documentclass{article}
\usepackage{ijcai21}

\RequirePackage[T1]{fontenc} 
\RequirePackage[tt=false, type1=true]{libertine} 
\RequirePackage[varqu]{zi4} 

\RequirePackage[libertine]{newtxmath}
\usepackage[letterspace=-25]{microtype}

\usepackage{soul}
\usepackage{url}
\usepackage[hidelinks]{hyperref}
\usepackage[utf8]{inputenc}
\usepackage[small]{caption}
\usepackage{graphicx}
\usepackage{amsmath}
\usepackage{booktabs}
\usepackage{savesym}
\savesymbol{Bbbk}
\savesymbol{openbox}

\usepackage{amssymb}
\usepackage{amsthm}
\usepackage{MnSymbol}
\usepackage{color}
\usepackage{url}
\usepackage{multirow}
\usepackage{xcolor}
\usepackage{dsfont}

\restoresymbol{NEW}{Bbbk}
\restoresymbol{NEW}{openbox}
\usepackage{amsfonts}
\usepackage{float}
\usepackage{tikz}
\usepackage{stackengine}
\usepackage{algorithm}
\usepackage[noend]{algpseudocode}
\usepackage{bm}
\usepackage{subcaption}
\usepackage{tablefootnote}
\usepackage{pgfplots}
\pgfplotsset{compat=1.10}
\usetikzlibrary{patterns}
\usetikzlibrary{shapes.misc}
\tikzset{cross/.style={cross out, draw=black, minimum size=2*(#1-\pgflinewidth), inner sep=0pt, outer sep=0pt},
cross/.default={2pt}}
\usepackage{multirow}
\newcommand{\model}[1]{\mathcal{M}_{#1}}

\newcommand{\abst}[2]{\mathcal{I}_{(#1,#2)}}

\newcommand{\card}[1]{\lvert#1\rvert}

\DeclareMathOperator*{\argmin}{arg\,min}

\newcommand{\mshift}{S}
\newcommand{\dist}[3]{d_{#1}(#2,#3)}
\newcommand{\distance}[3]{\lVert #1 - #2\rVert_{#3}}

\newcommand{\name}{\Delta}

\newcommand{\dataset}{\mathcal{D}}

\newcommand{\deleteFT}[1]{\textcolor{pink}{%$\times$ #1
}}

\algnewcommand\algorithmicforeach{\textbf{for each}}
\algdef{S}[FOR]{ForEach}[1]{\algorithmicforeach\ #1\ \algorithmicdo}

\usepackage[switch]{lineno}
\usepackage{tablefootnote}

\definecolor{cadmiumgreen}{rgb}{0.0, 0.75, 0.40}
\definecolor{richcarmine}{rgb}{0.85,0.00,0.23}
\newtheorem{definition}{Definition}
\newtheorem{proposition}{Proposition}

\usepackage{fancyhdr}
\fancypagestyle{firstpage}{%
    \fancyhf{}
    \setlength{\headsep}{0.3in}
    \cfoot{\thepage}
}
\fancypagestyle{otherpages}{%
    \fancyhf{}
    \setlength{\headsep}{0.3in}
    \rhead{Paper Title}
    \lhead{Authors}
    \cfoot{\thepage}
}
\title{Provably Robust and Plausible Counterfactual Explanations for \\Neural Networks via Robust Optimisation}

\author {
        Junqi Jiang$^1$,
        Jianglin Lan$^2$, Francesco Leofante$^1$, Antonio Rago$^1$, Francesca Toni$^1$\\ \vspace{0.2cm}
\affiliations\large
    $^1$ Department of Computing, Imperial College London 
    \\$^2$ James Watt School of Engineering, University of Glasgow\\ \vspace{0.2cm}
\emails
    \{junqi.jiang, f.leofante, a.rago, f.toni\}@imperial.ac.uk,  jianglin.lan@glasgow.ac.uk
}

\begin{document}
\maketitle

\begin{abstract}
	Counterfactual Explanations (CEs) have received increasing interest as a major methodology for explaining neural network classifiers. Usually, CEs for an input-output pair are defined as data points with minimum distance to the input that are classified with a different label than the output.
To tackle the established problem that CEs are easily invalidated when model parameters are updated (e.g. retrained), %due to retraining, 
studies have proposed ways to certify the robustness of CEs under model parameter changes bounded by a norm ball. However, existing methods targeting this form of robustness are not sound or complete, and they may generate implausible CEs, i.e., outliers wrt the training dataset. In fact, no existing method simultaneously optimises for proximity and plausibility while preserving robustness guarantees. In this work, we propose Provably RObust and PLAusible Counterfactual Explanations (PROPLACE)\footnote{The implementation is available at https://github.com/junqi-jiang/proplace}, a %robust optimisation framework leveraging on Mixed Integer Linear Programming
method leveraging on robust optimisation techniques to address the aforementioned limitations in the literature. 
%By utilising the robustness certification, we carefully construct a plausible CEs search space such that the soundness and completeness of the method are conditionally guaranteed. 
We formulate an iterative algorithm to compute provably robust CEs and prove its %\delete{termination} \JL
{convergence}, soundness and completeness. Through a comparative experiment involving six %five %existing robust CEs methods 
baselines, five of which target robustness, we show that PROPLACE achieves state-of-the-art performances against metrics on three evaluation aspects.
\end{abstract}

\section{Introduction}
\label{sec:intro}
%\begin{itemize}
%    \item Improvements from AAAI: complete, plausible
%    \item improvements from ROAR: proximity, plausible
%    \item the only method that optimises $\Delta$ robustness that is also plausible.
%\end{itemize}

%\todo{decide what to do about saying: methods targeting Delta robustness are also robust to unbounded parameter changes.}

%\todo{is it robust optimisation or bi-level optimisation? - Bi-Level Robust Optimisation}

%\todo{change to robust optimisation and bi-level optimisation}

Counterfactual Explanations (CEs) have become a major methodology to explain NNs due to their simplicity, compliance with the regulations {\cite{Wachter_17}}, and alignment with human thinking {\cite{Celar2023}}.  %When provided to the end users of machine learning classifiers, CEs could also serve as guidelines for taking future actions. 
Given an input point to a classifier, a CE is a modified input classified with another, often more desirable, label. %\todo{use other examples? Loan is very old-fashioned} 
Consider a customer that is denied a loan by the machine learning system of a bank. A CE the bank provided for this customer could be, \emph{the loan application would have been approved, had you raised your annual salary by {\$} 6000.} Several desired properties of CEs have been identified in the literature%\delete{. T}\AR
{, t}he most fundamental %\delete{one}\AR
{of which} is \emph{validity}, {requiring} that the CE needs to be correctly classified with a specified label {\cite{Tolomei_17}}. \emph{Proximity} refers to the closeness between the CE and the input measured by some distance metric, which translates to {a measure of} the effort the end user has to make to achieve the prescribed changes {\cite{Wachter_17}}. The CEs should also lie on the data manifold of the training dataset and not be an outlier, which is %\delete{referred to as}\AR
{assessed via} \emph{plausibility} \cite{poyiadzi2020face}. %\todo{AR: we could exemplify each of these in the loan (or whatever else we choose) example?}
Most recently, the \emph{robustness} of CEs, amounting to their validity under various types of uncertainty, has drawn increasing attention due to its real-world importance. In this work, we consider {robustness to} the model parameter changes occurring in the classifier on which the CE was generated. Continuing the loan example, assume the bank's machine learning %\delete{system gets}\AR
{model is} retrained with new data, while{, in the meantime,} the customer has achieved a raise in salary {(as prescribed by the CE)}. The customer {may then return} %\delete{comes back} 
to the bank only to find {that} the previously specified CE is now invalidated 
%\delete{on}\AR
{by} the new model. In this case, the bank %\delete{will}\AR
{could} be %\AR
{seen as being} responsible %\delete{for}\AR
{by} the user %\delete{'s frustration which will} \AR
{and could} potentially %\AR
{be legally liable, risking financial and reputational} %\delete{cause} 
damage to the %\delete{value of the} 
organisation. The quality of such unreliable CE is also questionable{:} \cite{rawal2020algorithmic,pmlr-v162-dutta22a} have shown that CEs found by existing non-robust methods are prone to such invalidation due to their closeness to the decision boundary. 

Various methods have been proposed to tackle this issue. %\delete{\cite{nguyen2022robust,pmlr-v162-dutta22a,hamman2023robust} \delete{focuses} \JL{focus} on building heuristic methods using model confidence, Lipschitz continuity, and quantities related to the data distribution, while \cite{upadhyay2021towards,blackconsistent,oursaaai23} consider optimising for the validity of CEs under bounded model parameter changes which are also empirically robust to the unbounded parameter changes scenarios.} \JL
{\cite{nguyen2022robust,pmlr-v162-dutta22a,hamman2023robust} focus on building heuristic methods using model confidence, Lipschitz continuity, and quantities related to the data distribution. \cite{upadhyay2021towards,blackconsistent,oursaaai23} consider optimising the validity of CEs under bounded model parameter changes, which are also empirically shown to be robust to the unbounded parameter changes scenarios.} %\delete{Among the existing methods, only \cite{oursaaai23} touches on formal robustness guarantees. Such formal properties are lacking in the explainable AI literature and are being advocated \cite{Marques-Silva_22}.} \JL
Among the existing methods, only \cite{oursaaai23} %\delete{touches on formal} \AR
{provides} robustness guarantees %\AR
{in a formal approach, which are known to be} %\delete{that are} 
lacking in the explainable AI (XAI) literature %\AR
{in general, aside from some notable examples, e.g. as introduced in}  %\delete{except being advocated in} 
\cite{Marques-Silva_22}. %\delete{The authors proposed} \JL{The work \cite{Marques-Silva_22} 
%\cite{oursaaai23} {presents} a method to certify  %The authors apply interval abstractions \cite{PrabhakarA19} of NNs to certify 
%whether a CE is valid for all possible model changes within a norm ball in the model parameter space \todo{what's wrong with this sentence?}. 
%\delete{However, their method to generate such provably robust CEs is not sound or complete, as it only tunes the hyperparameters of a non-robust CEs method in an iterative procedure and repeatedly tests for robustness. There is no guarantee that a provably robust solution can be found, which is also the case for \cite{upadhyay2021towards}.} \JL
{Their method generates such provably robust CEs via iteratively tuning the hyperparameters of an arbitrary non-robust CEs method and testing for robustness. {However, t}his method cannot always guarantee soundness and is not complete, which is also the case for the method in \cite{upadhyay2021towards}.} Another limitation in the current literature is that the methods targeting this form of robustness guarantee do not find plausible CEs, %\delete{making them less applicable in practice} \JL
{limiting their practical applicability}. 

Such limitations have motivated this work.
After discussing relevant studies in Section~\ref{sec:related}, we introduce 
the robust optimisation problem for computing CEs with proximity property as the objective, %\JL
{and} robustness and plausibility properties as constraints (Section~\ref{sec:problemstatement}). In Section~\ref{sec:method}, we then present Provably RObust and PLAusible CEs (PROPLACE), a method leveraging on robust %\todo{or robust?} \JL{JL: how about bi-level robust optimisation?} 
optimisation techniques %, a common approach for solving min-max objectives \cite{shaham2018understanding},
to address the limitation in the literature that no method optimises for proximity and plausibility while providing formal robustness guarantees. We show the (conditional) soundness and completeness of our method, and give a bi-level optimisation procedure that will converge and terminate. Finally, in our experiments, we compare PROPLACE with %\delete{five}\AR
{six} {existing} %\delete{robust} 
CE methods%\AR
{, five of which target robustness,} on four benchmark datasets. %\delete{We show empirically that} \JL
{The results show that} our method achieves the best robustness and plausibility, %\delete{without {sacrificing} proximity} 
while demonstrating superior proximity among the most robust baselines.

\section{Related Work}
\label{sec:related}

%\todo{AR: to be explained beforehand? background on CEs generally/rough example of what one is, validity, proximity, plausibility, robustness to model changes, data manifold
%JJ: Done in Intro}

As increasing interest has been focused on XAI, a plethora of CE generation methods have been proposed (see \cite{Guidotti_22} for a recent overview). Given our focus on neural networks, we cover those explaining the outputs of these models. %\cite{Tolomei_17} first proposed CEs as a recourse method to achieve different classification results by tree-based models. 
\cite{Wachter_17} proposed a gradient-based optimisation method targeting the validity and proximity of CEs. Similarly, using the %\delete{MILP} \JL
{mixed integer programming (MILP)} representation of neural networks, \cite{mohammadi2021scaling} formulated the CEs search into {a} constrained optimisation problem such that the resulting CEs are guaranteed to be valid. \cite{mothilal2020explaining} advocated generating a diverse set of CEs for each input 
to enrich the information provided to the explainee. 
Several works also addressed \emph{actionability} constraints \cite{ustun2019actionable,verma2022amortized,vo2023feature}, only allowing changes in the actionable features of real users.
\cite{poyiadzi2020face} proposed a graph-based method to find a path of CEs that are all lying within the data manifold. Several other works have proposed to use ({v}ariational) {a}uto-{e}ncoders or nearest neighbours to induce plausibility~\cite{dhurandhar2018explanations,pawelczyk2020learning,van2021interpretable}. Among these properties, {actionability and plausibility are two orthogonal considerations which make the CEs realistic in practice, and} trade-offs have been identified between plausibility and proximity \cite{pawelczyk2020counterfactual}.
%There are other methods that study CEs in a causal setting where parts or all of a structural causal model are available \AR{(}\cite{karimi2020algorithmic,Dominguez-Olmedo_22}\AR{)}. 
%\delete{Various methods have been proposed to generate CEs for other types of classifiers, we refer readers to \cite{Guidotti_22} for a recent overview.}

In this work, our focus is on the {property of} robustness to {changes in} the model parameter{s, i.e. weights and biases} in the underlying classifier. {Several} studies looked at CEs under bounded model parameter changes of a neural network: \cite{upadhyay2021towards} %\delete{formulate} \JL
{formulated} a novel loss function and solve{d} using gradient-based methods. 
\cite{blackconsistent} proposed a heuristic based on the classifier's Lipschitz constant and the model confidence to search for robust CEs. \cite{oursaaai23} %\delete{uses} \JL
{used} interval abstraction{s} %\delete{\cite{PrabhakarA19}} \JL
{\cite{PrabhakarA19}} to certify the robustness against bounded parameter changes, and embed the certification process into existing CE methods. Differently to our approach, these methods do not generate plausible CEs {or} guarantee {that} provably robust CEs {are found}. Other relevant works place their focus on the robustness of CEs against {unbounded model changes}. \cite{robovertime} took the approach of {augmenting the} training data with previously generated CEs. \cite{nguyen2022robust} focused on the data distribution and formulated the problem as posterior probability ratio minimisation to generate robust and plausible CEs.  
By using first- and second-moment information, \cite{bui2022counterfactual} proposed lower and upper bounds on the CEs' validity under random parameter updates and generated robust CEs using gradient descent.
\cite{pmlr-v162-dutta22a} defined a novel robustness measure based on the model confidences over the neighbourhood of the CE, and used dataset points that satisfy some robustness test to find close and plausible CEs. Their notion is then further re-calibrated for neural networks with probabilistic robustness guarantees in \cite{hamman2023robust}. Trade-offs between robustness and proximity were discussed by \cite{pawelczyk2022probabilistically} and \cite{upadhyay2021towards}.
% \JJ{Note that, %\delete{through}\AR
% {as seen in} the experiment results of \cite{upadhyay2021towards,blackconsistent,oursaaai23}, CEs robust to bounded model changes are also robust to unbounded changes.}

Other forms of CEs' robustness have also been investigated, for example, robustness against: 
input perturbations \cite{Alvarez-Melis_18,Sharma_20,gnnce,Dominguez-Olmedo_22,Huai_22,Virgolin_23,Zhang_23};  noise in the execution of CEs 
\cite{pawelczyk2022probabilistically,LeofanteLomuscio23,LeofanteLomuscio23b,maragno2023finding}; 
and model multiplicity 
\cite{pawelczyk2020counterfactual,Leofante2023modelmultiplicity}.

\section{Preliminaries and Problem Statement}
\label{sec:problemstatement}
\paragraph{Notation.} Given an integer $k$, we use $[k]$ to denote the set $\{1,\ldots,k\}$. We use $\card{S}$ to denote the cardinality of a set $S$. 

\paragraph{%\delete{Neural Networks} \JL
{Neural Network (NN)}.} We denote a NN as $\model{\Theta} : \mathcal{X}\subseteq \mathbb{R}^d \rightarrow \mathcal{Y} \subseteq \mathbb{N}$, where the inputs are $d$-dimensional vectors and the outputs are discrete class labels. $\Theta$ represents the collection of parameters that characterise the NN. %The final component of the parametric neural networks is usually a Softmax function (or Sigmoid function for binary classification tasks) which re-interpret the previous-layer output for each class as a class probability.
Throughout the paper, we will illustrate our method using the binary classification case (i.e. $\mathcal{Y}=\{0,1\}$), though the method is readily applicable to multi-class classification.
Let $\model{\Theta}(x)$ also (with an abuse of notation) refer to the pre-sigmoid (logit) value {in} the NN.
Then, 
for an input $x \in \mathcal{X}$, we say $\model{\Theta}$ classifies $x$ {as} class $1$ if $\model{\Theta}(x) \geq 0$, otherwise $\model{\Theta}$ classifies $x$ {as} class $0$. 

\paragraph{%\delete{Counterfactual Explanations} \JL
{Counterfactual Explanation (CE)}.} For an input $x \in \mathcal{X}$ that is classified to the unwanted class $0$  (assumed throughout the paper), a CE $x' \in \mathcal{X}$ is some other data point ``similar'' to the input, e.g. by some distance measure, but classified to the desired class $1$. 

\begin{definition}{\emph{(CE)}}
Given a NN $\model{\Theta}$, an input $x\in\mathcal{X}$ such that $\model{\Theta}(x)<0$, and a distance metric $dist: \mathbb{R}^d \times \mathbb{R}^d \rightarrow \mathbb{R}^+$, a \emph{CE} $x'\in \mathcal{X}$ is such that:
\begin{alignat*}{3}
&\argmin_{x'}  && \text{  } dist(x,x')\\%\label{eqn:obj} \\
&\text{subject to} && \quad \model{\Theta}(x') \geq 0
%,\\&\text{ } && \quad x' \in \mathcal{X}. %\label{eqn:eq-1}
\end{alignat*}
\label{eqn:all-lines}
%\todo{check all equation numbers are needed 
%JJ: Done}
\label{def:ce}
\end{definition}

%\FL{mention one example of dist in the literature, e.g. L1}
 %Depending on the characteristics of the dataset, one could also consider L2 distance \cite{blackconsistent} or customised distance according to user preferences \cite{upadhyay2021towards}.}
The minimum distance objective {targets the} minimum effort by the end user to achieve a change, which corresponds to the basic requirement of proximity mentioned in Section~\ref{sec:intro}. In the literature, normalised $L_1$ 
distance is often adopted as the distance metric because it induces changes in fewer features in the CE \cite{Wachter_17}. However, methods that find such plain CEs usually result in unrealistic combinations of features, or outliers to the underlying data distribution of the training dataset. A plausible CE avoids these issues {and is} formally {defined as follows}:

\begin{definition}{\emph{(Plausible CE)}}
Given a NN $\model{\Theta}$ and an input $x\in\mathcal{X}$ such that $\model{\Theta}(x)<0$, a distance metric $dist: \mathbb{R}^d \times \mathbb{R}^d \rightarrow \mathbb{R}^+$ and some plausible region $\mathcal{X}_{plaus} \subseteq \mathbb{R}^d$, a \emph{plausible CE} is an $x'$ such that:
\nonumber
\begin{alignat*}{3}
&\argmin_{x'}  && \text{  } dist(x,x')\\ %\label{eqn:objplaus} \\
&\text{subject to} && \quad \model{\Theta}(x') \geq 0, \quad x' \in \mathcal{X}_{plaus}\\ %\label{eqn:plausobjmodel}\\
%& \text{ }&& \quad x' \in \mathcal{X}_{plaus}\\ %\label{eqn:eq-1plaus}
\end{alignat*}
%\label{eqn:all-linesplaus}
%\end{subequations}

\label{def:plausiblece}
\end{definition}

The plausible region $\mathcal{X}_{plaus}$ {may be used to} eliminate any unrealistic feature values (e.g. a value {of} 0.95 for a discrete feature), {or to} indicate a densely populated region that is close to the data manifold of the training dataset. Additionally, it {may} also include some actionability considerations{, such as} restricting immutable attributes ({e.g. avoiding suggesting} changes in gender) {or} specifying some relations between input features ({e.g.} obtaining a doctoral degree should also cost the user at least 4 years). 

\paragraph{Robustness of Counterfactual Explanations.} Studies have shown that CEs found by the above formulations are readily invalidated when small changes occur {in} the model parameters of the NNs. We formalise this in the following and {begin} by introducing {a} distance {measure} between two NNs and {a definition of} model shift. Note that Definitions~\ref{def:modeldistance} to \ref{def:delta_robustness} are adapted from \cite{oursaaai23}.

%\FL{mention that the following definition are adapted from Jiang et al 2023}

\begin{definition}{\emph{(Distance between two NNs)}}
    Consider two NNs $\model{\Theta}$, $\model{\Theta'}$ of the same %model 
    architecture characterised by parameters $\Theta$ and $\Theta'$. For $0\leq p\leq \infty$, the \emph{p-distance} between $\model{\Theta}$ and $\model{\Theta'}$ is $\dist{p}{\model{\Theta}}{\model{\Theta'}} = \distance{\Theta}{\Theta'}{p}$.
\label{def:modeldistance}
\end{definition}

%\todo{remove the definition of model shift, use shifted model $\model{\Theta'}$ directly?}

%\begin{definition}
%Given $0 \!\leq\! p\! \leq\!\! \infty$, a \emph{model shift} is a function $\mshift$ mapping a neural network $\model{\Theta}$ into another $\model{\Theta'} = \mshift(\model{\Theta})$ such that:   
%\begin{itemize}
%    \item $\model{\Theta}$ and $\model{\Theta'}$ have the same model structure;
%    \item $\dist{p}{\model{\Theta}}{\model{\Theta'}} > 0$.
%\end{itemize}
%\label{def:model_shift}
%\end{definition}

\begin{definition}{\emph{(Bounded model shifts)}}
Given a NN $\model{\Theta}$, $\delta \in \mathbb{R}_{>0}$ and $0 \leq p \leq \infty$, the \emph{set of %\delete{plausible shifted models} \AR
{bounded model shifts}} is defined as
%$\Delta = \{ \mshift \mid \dist{p}{\model{\Theta}}{\mshift(\model{\Theta})} \leq \delta\}$. 
$\Delta = \{\model{\Theta'} \mid \dist{p}{\model{\Theta}}{\model{\Theta'}} \leq \delta\}$. 
\label{def:set_of_plausible_shifts}
\end{definition}

%\delete{Note that the definitions apply because the two model parameter vectors $\Theta$ and $\Theta'$ have the same shape due to the requirement that the two models are of identical model structures.} \todo{if we want to keep this it needs to be rephrased as ``the definitions apply'' isn't clear} \JJ{JJ: we don't really need this, deleted}

\paragraph{Certifying Robustness.} 
Having presented the definitions required to formalise the optimisation problem for finding provably robust and plausible CEs, we now introduce another relevant technique that uses interval abstractions to certify the robustness of CEs. %We instantiate the definitions for $p=\infty$ such that the interval abstraction exactly encodes all the possible shifted models in $\Delta$ \cite{oursaaai23}. 
We refer to the certification process as the $\Delta$-robustness test; this will be used for parts of our method and also as an evaluation metric in the experiments. We assume $p=\infty$ for bounded model shifts $\Delta$ throughout the paper.  

%\FL{the second bullet in the definition below is true only for $p=\infty$}
\begin{definition}{\emph{(Interval abstraction of NN)}} 
Consider a NN $\model{\Theta}$ with $\Theta = [\theta_0, \ldots, \theta_d]$. Given a set of %\delete{plausible shifted models}\AR
{bounded model shifts} $\Delta$, we define the \emph{interval abstraction of $\model{\Theta}$ under $\Delta$} as the model $\abst{\Theta}{\Delta}: \mathcal{X} \rightarrow \powerset{\mathbb{R}}$ (for $\powerset{\mathbb{R}}$ the set of all closed intervals over $\mathbb{R}$) such that:
\begin{itemize}
\item $\model{\Theta}$ and $\abst{\Theta}{\Delta}$ have the same architecture; 
\item $\abst{\Theta}{\Delta}$ is parameterised by an interval-valued vector $\bm{\Theta} \!\!=\!\! [\bm{\theta}_0, \!\ldots, \! \bm{\theta}_d]$ such that, for $i \in \{0,\ldots,d\}$,
%$\bm{\theta}_i \in \powerset{\mathbb{R}}$;
% $\bm{\theta}_i$  encodes the range of possible changes induced by the application of any $\mshift{} \in \Delta$ to $\model{\Theta}$ such that 
$\bm{\theta}_i = [\theta_i - \delta,  \theta_i + \delta]$, where $\delta$ is the bound %for model parameter changes 
in $\name$.
\end{itemize}
\label{def:interval_abstraction}
\end{definition}

%Intuitively, $\bm{\theta}_i$  encodes the range of possible changes by the application of $\Delta$ to $\model{\Theta}$. 

When $p=\infty$, $\bm{\theta}_i$ encodes the range of possible model parameter changes by the application of $\Delta$ to $\model{\Theta}$. Given a fixed input, by propagating the weight and bias intervals, the output range of $\abst{\Theta}{\Delta}$ exactly represents the possible output range for the input by applying $\Delta$ to $\model{\Theta}$ \cite{oursaaai23}.

\begin{definition}{\emph{(Interval abstraction of NN classification)}}
    Let $\abst{\Theta}{\Delta}$ be the interval abstraction of a NN $\model{\Theta}$ under $\Delta$. Given an input $x \in \mathcal{X}$, let $\abst{\Theta}{\Delta}(x)=[l,u]$% be the output interval obtained by applying $\abst{\Theta}{\Delta}$ to $x$
    . Then, we say that \emph{$\abst{\Theta}{\Delta}$ classifies $x$ as class $1$} if $l \geq 0$ (denoted, with an abuse of notation,  $\abst{\Theta}{\Delta}(x) \geq 0$),  and as class $0$ if $u<0$ (denoted, with an abuse of notation,  $\abst{\Theta}{\Delta}(x) < 0$). %\todo{this is a bit weird, doesn't $\abst{\Theta}{\Delta}(x) = [l,u]$ in the previous def?} \JJ{JJ: could we say with an abuse of notation?} \JL{I think the current definition is easy to understand.}
    \label{def:classification_interval}
\end{definition}

Indeed, for an input, if the lower bound $l$ of pre-sigmoid output node interval $[l, u]$ of $\abst{\Theta}{\Delta}$ satisfies $l \geq 0$, then it means all shifted models in $\Delta$ would predict the input with a pre-sigmoid value that is greater than or equal to 0, all resulting in predicted label 1. We apply this intuition to the CE context:

\begin{definition}{\emph{($\name$-robust CE)}}
    Consider an input $x \in \mathcal{X}$ and a model $\model{\Theta}$ such that $\model{\Theta}{(x)} < 0$. Let $\abst{\Theta}{\Delta}$ be the interval abstraction of $\model{\Theta}$ under $\Delta$. We say that \emph{a CE $x'$ is $\name$-robust} iff $\abst{\Theta}{\Delta}(x') \geq 0$.
    \label{def:delta_robustness}
\end{definition}

Checking whether a CE $x'$ is $\Delta$-robust requires {the calculation of} the lower bound $l$ of the pre-sigmoid output node interval $[l, u]$ of $\abst{\Theta}{\Delta}$. This process can be encoded as a MILP program (see Appendix B in \cite{oursaaai23}). %To conclude, passing the $\Delta$-robustness test via the interval abstraction technique means that a CE $x'$ would remain valid under all the possible model shifts included in $\Delta$. %\JL{(JL: It is strange to mention MILP without further elaboration. Shall we give a reference or refer to the MILP section later presented?)} \JJ{JJ: Reference to AAAI paper should be enough? }

\paragraph{Optimising for Robustness and Plausibility.} Now we introduce the targeted provably robust and plausible optimisation problem based on Definitions \ref{def:plausiblece} and \ref{def:delta_robustness}, by taking inspiration from the robust optimisation technique \cite{Ben+09}.

\begin{definition} {\emph{(Provably robust and plausible CE)}}
Given a NN $\model{\Theta}$, an input $x\in\mathcal{X}$ such that $\model{\Theta}(x)<0$, a distance metric $dist: \mathbb{R}^d \times \mathbb{R}^d \rightarrow \mathbb{R}^+$ and some plausible region $\mathcal{X}_{plaus} \subseteq \mathbb{R}^d$, let $\abst{\Theta}{\Delta}$ be the interval abstraction of $\model{\Theta}$ under the bounded model shifts $\Delta$. Then, a \emph{{p}rovably {r}obust and {p}lausible CE} $x' \in \mathcal{X}$ is such that:
\begin{subequations}
\begin{alignat}{3}
&\argmin_{x'}  && \text{  } dist(x,x')\label{eqn:proximityobj} \\
&\text{subject to} && \quad \abst{\Theta}{\Delta}(x')\geq 0, \label{eqn:robobj}\\
&\text{ }  && \quad x' \in \mathcal{X}_{plaus} \label{eqn:plausobj}
\end{alignat}
\label{eqn:original_robust_problem}
\end{subequations}
\label{def:robandplausproblemminmax}
\end{definition}

%We highlight that problem formulation in Equations \ref{eqn:proximityobj} to \ref{eqn:plausobj} is a bi-level optimisation problem and is not MIP-solvable \todo{How to prove this?} The constraint in Equation~\ref{eqn:robobj} itself requires a MIP program to compute. 
%\delete{Equations~\ref{eqn:proximityobj} - \ref{eqn:plausobj} in the optimisation problem \eqref{def:robandplausproblemminmax}} \JL
The optimisation problem %\delete{\eqref{eqn:proximityobj} - \eqref{eqn:plausobj}}}
{\eqref{eqn:original_robust_problem}} can be {equivalently} rewritten as follows:
\begin{subequations}
\begin{alignat}{3}
&\argmin_{x'}  && \text{  } dist(x,x')\label{eqn:proximityobjminmax} \\
&\text{subject to} && \quad \max_{\model{\Theta'} \in \Delta} [-\model{\Theta'}(x')] \leq 0,\label{eqn:robobjminmax}\\
%\max_{\model{\Theta'} \in \Delta} [-\model{\Theta'}(x') + \sigma] \leq t,\label{eqn:robobjminmax}\\
&\text{ } && \quad x' \in \mathcal{X}_{plaus} \label{eqn:plausobjminmax}
\end{alignat}
\label{eqn:minmax_robust_problem}
\end{subequations}

%\label{def:robandplausproblemminmax}
%\end{definition}

%This paper aims to develop
We  show  next a novel approach for solving this robust optimisation problem \eqref{eqn:minmax_robust_problem}% {targeting} the robust and plausible CEs
.

\section{PROPLACE}
\label{sec:method}

\begin{algorithm}[b]
\caption{PROPLACE} \label{alg:algo1}
\begin{algorithmic}[1]
\Require input $x$, model $\model{\Theta}$, 
\State $\quad$ training dataset $\dataset = \{(x_1,y_1),\ldots,(x_n,y_n)\}$, 
\State $\quad$ set of bounded model shifts $\Delta$,
\State $\quad$ plausible region to be used as CEs search space $\mathcal{X}_{plaus}$.
\State \textbf{Init}: $x' \gets \emptyset$; $\Delta' \gets {\{\model{\Theta}\}}$

%\Repeat{$(-\model{\Theta'}(x')) \leq 0$}
\State{\textbf{Repeat until }$(-\model{\Theta'}(x')) \leq 0$}

{
\quad $x'\gets \texttt{Outer\_minimisation}(\model{\Theta}, x, \mathcal{X}_{plaus}, \Delta')$

\quad $\model{\Theta'} \gets \texttt{Inner\_maximisation}(x', \Delta', \Delta)$

\quad $\Delta' \gets \Delta' \cup \{\model{\Theta'}\}$
}
\State \textbf{return} $x'$
\end{algorithmic}
\end{algorithm}

%\AR{In this section, we will introduce our} \delete
{The} procedure for computing robust and plausible CEs{, solving the optimisation problem %\JL
{\eqref{eqn:minmax_robust_problem}} 
%\delete{which was stated in Definition~\ref{def:robandplausproblemminmax}\AR{, as} \delete{is}}\JL{
, is}  
summarised in Algorithm~\ref{alg:algo1}. We will first introduce how the plausible region $\mathcal{X}_{plaus}$ is constructed in Section \ref{subsec:id search space} (corresponding to Line 3, Algorithm~\ref{alg:algo1}). %\delete{Section \ref{subsec:id search space} will also include details of how the search space for CEs is reduced such that the resulting CEs are plausible (Equation~\ref{eqn:plausobjminmax}), and that the completeness of the algorithm can be guaranteed. } 
Then, {in Section \ref{subsec:bilevel optimisation}} we {will} present the bi-level optimisation method (corresponding to Lines 4-5, Algorithm~\ref{alg:algo1}) %\delete{in Section \ref{subsec:bilevel optimisation}}\
to solve the robust optimisation problem \eqref{eqn:minmax_robust_problem}. 
%\delete{Finally,}
{In Section \ref{subsec:bilevel optimisation}} we {will also} instantiate the complete bi-level optimisation formulations (in {MILP} form) of our method for NNs with ReLU activation functions. %\delete{in Section \ref{subsec:MILP form}}. 
{Finally, i}n Section \ref{subsec:theoreticalstuff} we discuss the soundness and completeness of Algorithm~\ref{alg:algo1} and prove its convergence.

%In this section, we first introduce how the search space for CEs is reduced such that the resulting CEs are plausible (Equation~\ref{eqn:plausobjminmax}), and that the completeness of the algorithm can be guaranteed. Then, we present the method to solve the min-max optimisation problem with the closeness objective and the robustness constraint (Equations~\ref{eqn:proximityobjminmax} and \ref{eqn:robobjminmax}). Finally, we present the MILP formulations for our method. %Then, the process to construct the plausible region (Equation~\ref{eqn:plausobjminmax}) is introduced. 

\subsection{Identifying Search Space $\mathcal{X}_{plaus}$}\label{subsec:id search space}
%\todo{address robustness properties for the plausible region. Perhaps it is better not be called plausibility constraint?}
%To induce plausibility, points from the training dataset that satisfy the basic CE requirement $\model{\Theta}(x_i)\geq0$ are frequently utilised in the literature. Specifically, they are either used as the final CEs, or as starting points to perform some extra algorithms to find CEs. In this work, w
%Data points from the training dataset have been frequently utilised for inducing plausibility of the resulting CEs.

As mentioned in Section~\ref{sec:related}, points from the training dataset (especially $k$-nearest-neighbours) are frequently utilised in the literature to induce plausibility. In this work, we propose to use a more special{ised} kind of dataset point, \emph{$k$ $\Delta$-robust {n}earest-{n}eighbours}, to construct the search space for CEs that is both plausible and robust. 

%As mentioned in Section~\ref{sec:related}, points from the training dataset are frequently utilised in the literature to induce plausibility. In our robust optimisation formulation, we also want to compute points that are both close to the input and are robust to model shifts $\Delta$. Therefore, we apply the $\Delta$-robustness test (Definition~\ref{def:delta_robustness}) over the dataset points to find the $\Delta$-robust ones, from which we then find the $k$ nearest neighbours to construct the robust and plausible region.

\begin{definition}{\emph{(k $\Delta$-robust nearest-neighbours)}}
%\emph{($k$ $\Delta$-robust nearest-neighbours)} \todo{all the definitions should be aligned in their convention, e.g.:} 
{Given a NN $\model{\Theta}$ and an input $x\in\mathcal{X}$ such that $\model{\Theta}(x)<0$, a distance metric $dist: \mathbb{R}^d \times \mathbb{R}^d \rightarrow \mathbb{R}^+$,}
a dataset $\dataset \subseteq \mathbb{R}^d$ on which  $\model{\Theta}$ is trained, and a set of %\delete{plausible}\AR
{bounded} model shifts of interest $\Delta$, let $\abst{\Theta}{\Delta}$ be the interval abstraction of $\model{\Theta}$ under $\Delta$. %\delete{Let $dist$ be a distance metric: $dist: \mathbb{R}^d \times \mathbb{R}^d \rightarrow \mathbb{R}^+$.} 
Then, the \emph{k $\Delta$-robust nearest-neighbours} of $x$ is a set $S_{k, \Delta} \subseteq \dataset$ with cardinality $|S_{k, \Delta}|=k$ such that:
\begin{itemize}
    \item $\forall x'\in S_{k, \Delta}$, $x'$ is $\Delta$-robust, i.e. $\abst{\Theta}{\Delta}(x')\geq 0$, %\todo{should this bullet not come first? JJ: good point, done.}
    \item $\forall %\delete{(x^{''})}\JL
    x^{''} \in \dataset \setminus S_{k, \Delta}$, if $x^{''}$ is $\Delta$-robust, $dist(x, x^{''}) \geq \underset{x'\in S_{k, \Delta}}{\emph{max}} \text{ }%\max_{x'\in S_{k, \Delta}}
    dist(x, x')$.
\end{itemize}
\end{definition}

The first constraint 
%\delete{requires}\AR
enforces the $\Delta$-robustness, and the second states that the points contained in the set are the $k$ nearest points to the input $x$ amongst all the $\Delta$-robust dataset points. In practice, in order to compute the $k$ $\Delta$-robust nearest-neighbours, we fit a k-d tree on the dataset points that are classified to the desired class, then iteratively query the k-d tree for the nearest neighbour of an input, until the result satisfies the $\Delta$-robustness test (Definition~\ref{def:delta_robustness}).  %\todo{maybe add a touch more description about def 9? JJ: added at the front}

Restricting the CE search space within the convex hull of these robust neighbours will likely induce high plausibility (and robustness). However, because these points are deep within 
%\delete{the data manifold of the} 
parts of the training dataset that are classified to another class, they 
%\delete{can}\AR
{may} be far from the model's decision boundary, therefore 
%\delete{producing}\AR
{resulting in} large distances to the inputs. In fact, \cite{pmlr-v162-dutta22a,hamman2023robust} adopted similar robust nearest neighbours (using other notions of robustness tests) as the final CEs, and poor proximity 
%\delete{results were}\AR
{was} observed in their experiment results. They have also shown that finding CEs using line search between proximal CEs and these robust neighbours can slightly improve proximity.%, though 
%\delete{that will then depend}\AR {this inherently relies} on other CE methods.

In our case, since the validity of the CEs can be guaranteed from the optimisation procedures (Section~\ref{subsec:bilevel optimisation}), we expand the plausible search space across the decision boundary by taking the input into consideration, which is assumed to also be inside the data distribution.

%\delete{As indicated in the experiment results of the recent work \cite{pmlr-v162-dutta22a,hamman2023robust}, which uses \AR{a} similar \AR{method for generating} robust nearest neighbours as the final CEs, the distances between the CEs and inputs are not ideal \todo{needs a bit more explanation}. %, and performing a search from the robust nearest neighbours to the input allows improvements in the distance. 
%We therefore also take the input point $x$ into consideration when defining the plausible region.} %, assuming that any changes from $S_{k, \Delta}$ to $x$ are also convenient to the data subject of $x$.

%The robustness certification (Definition~\ref{def:delta_robustness}) is used as a filter to discover the dataset points that are in the feasible region of the robust optimisation problem (Equation~\ref{eqn:robobjminmax}), 

\begin{definition}
    \emph{(Plausible region)} Given an input $x\in \mathbb{R}^d$ and its k $\Delta$-robust nearest neighbours $S_{k, \Delta}$, the \emph{plausible region} $\mathcal{X}_{plaus}$ is the convex hull 
        %($\mathcal{H}$) of the set containing the aforementioned points:  $\mathcal{X}_{plaus} = \mathcal{H}(S_{k, \Delta} \cup \{x\})$.
        of $S_{k, \Delta} \cup \{x\}$.
\end{definition}

By restricting the CE search space to such convex hull, the method has the flexibility to find close CEs (with $x$ as a vertex), or robust and plausible CEs (with the robust neighbours as other vertices). This $\mathcal{X}_{plaus}$ ensures the soundness and completeness of our method (Section~\ref{subsec:theoreticalstuff}).

%\todo{say what this buys you?}

%Specifically, in the MILP formulation, the plausibility constraint $x' \in \mathcal{X}_{plaus}$ is equivalent to the following constraints:
%\begin{subequations}\label{OP:plausibility}
%	\begin{align}
%            \label{const8:plaus1}
%            & \lambda_l \in [0, 1], \sum_{l}^{|S_{k, \Delta} \cup \{x\}|}=1, \quad  l \in [|S_{k, \Delta} \cup \{x\}|]\\
%            \label{const9:plaus2}
%            & x' = \sum_{l}^{|S_{k, \Delta} \cup \{x\}|}\lambda_l x'_l, \quad  x'_l\in S_{k, \Delta} \cup \{x\}
%	\end{align}
%\end{subequations}

\subsection{Bi-level Optimisation Method with MILP}\label{subsec:bilevel optimisation}
\subsubsection{Outer and Inner Optimisation problems}\label{subsubsec:outer inner problems}
We separate the robust optimisation problem \eqref{eqn:minmax_robust_problem} to solve into outer minimisation and inner maximisation problems, as specified in Definitions~\ref{def:outerprob} and \ref{def:innerprob}. 

\begin{definition}
Given a %\delete{classification model} \todo{before we said neural network, this needs to be aligned throughout}
NN $\model{\Theta}$ and an input $x\in\mathcal{X}$ such that $\model{\Theta}(x)<0$, %\delete{specify} 
a distance metric $dist: \mathbb{R}^d \times \mathbb{R}^d \rightarrow \mathbb{R}^+$ and some plausible region $\mathcal{X}_{plaus} \subseteq \mathbb{R}^d$, let $\Delta'$ be a set of shifted models. Then, the \emph{outer minimisation problem} finds a CE $x'$ such that:
\begin{subequations}
\begin{alignat}{3}
&\argmin_{x'}  && \text{  } dist(x,x')\label{eqn:outerobj} \\
&\text{subject to} && \quad -\model{\Theta'}(x') \leq 0, \text{ for each } \model{\Theta'} \in \Delta', \label{eqn:outerrob}\\
&\text{ } && \quad x' \in \mathcal{X}_{plaus}%\JL{.} 
\label{eqn:outerplaus}
\end{alignat}
\label{eqn:outer_problem}
\end{subequations}
\label{def:outerprob}
\end{definition}

%The set $\Delta'$ consists of results from the inner minimisation executions, which find the shifted model that gives the worst-case model output for the previously computed CE.

%\FL{why are we mentioning I if it does not appear in def 12?}

\begin{definition}
Given a CE $x'\in\mathcal{X}$ %such that $\model{\Theta'}(x')\geq0$ for each $\model{\Theta'}\in\Delta'$, 
found by the outer minimisation problem, %let $\abst{\Theta}{\Delta}$ be the interval abstraction of $\model{\Theta}$ given a set of %\delete{plausible}\AR
the set of bounded model shifts $\Delta$, the \emph{inner maximisation problem} finds a shifted model $\model{\Theta'}$ such that:
\begin{subequations}
\begin{alignat}{2}
%&\argmax_{\model{\Theta'}} 
& \underset{\model{\Theta'}}{\emph{\text{arg max}}} && \text{  } -\model{\Theta'}(x') \label{eqn:innerobj} \\
&\text{subject to} && \quad \model{\Theta'} \in \Delta%\JL{.} \label{eqn:innermodel}
\end{alignat}
\label{eqn:inner_problem}
\end{subequations}

\label{def:innerprob}
\end{definition}

The outer minimisation problem relaxes the constraint that a CE should be robust to all possible model shifts in the set $\Delta$; instead, it requires robustness wrt a subset of the model changes $\Delta' \subset \Delta$. $\Delta'$ is initialised with the original classification model $\model{\Theta}$. At the first execution, the outer minimisation finds the closest CE $x'$ valid for that model. Then, $x'$ is passed to the inner maximisation problem to compute the model shift $\mshift({\model{\Theta})}$ that produces the lowest model output score. This model shift is considered to be the worst-case perturbation on the model parameters in the set $\Delta$, and is added to $\Delta'$. In the next iterations, $x'$ is updated to the closest CE valid for all the models in $\Delta'$ (outer), which is being expanded (inner), until convergence. 

\subsubsection{MILP Formulations}\label{subsec:MILP form}
The proposed bi-level optimisation method in Section \ref{subsubsec:outer inner problems} is independent of specific NN structures. In this section, we take NNs with ReLU activation functions as an example to further elaborate the method. We denote the total number of hidden layers in an NN $\model{\Theta}$ as $h$. We call $N_0$, $N_{h+1}$, and $N_i$  the %\delete{set}\JL
{sets} of input, output, and hidden layer nodes for $i \in [h]$, and their node values are $V_0$, $V_{h+1}$, and $V_i$. For hidden layer nodes 
$V_i = \text{ReLU}(W_i V_{i-1} + B_i)$, and for output layer nodes $V_{h+1}=W_{h+1} V_h  + B_{h+1}$, where $W_i$ is the weight matrix connecting nodes at layers $i-1$ and $i$, {and} $B_i$ is the bias vector of nodes $N_i$. %There are two MIP programs respectively for the outer and inner problems.
We instantiate the formulations using normalised $L_1$, while our method PROPLACE can accommodate arbitrary distance metrics.

%\begin{definition} A \emph{fully-connected feed-forward 
%neural network} is a tuple $\model{\Theta} = (h,N,E,B,\Omega)$ where:
%	\begin{itemize}
%		\item $h\geq 0$ is the depth of $\model{\Theta}$;
%		\item $(N,E)$ is a directed graph;
%		\item $N = \bigsqcup_{i=0}^{h+1} N_i$ is the disjoint union of sets of nodes $N_i$;
		%\item 
%		we call $N_0$ the input layer, $N_{h+1}$ the output layer and $N_i$ hidden layers for $i \in [h]$; 
%		\item $E \!=\! \bigcup_{i=1}^{h+1}(N_{i-1} \times N_i)$ is the set of edges %connecting subsequent 
%		between layers;
%		\item $B\!: \! (N \! \setminus \! N_0)\!\rightarrow\!\mathbb{R}$  assigns bias to nodes in non-input layers;
%		\item $\Omega\!: \!E \rightarrow \mathbb{R}$ assigns a weight to each 		edge. 
%	\end{itemize}
%	\label{def:ffnn}
%\end{definition}

%\begin{definition} Given an input $x \in \mathbb{R}^{\card{N_0}}$, an FFNN 
%$\model{\Theta}$ computes an \emph{output} %$\model{\Theta}(x)$ defined as %: 
%	follows. Let:
%	\begin{itemize}
%		\item $V_0 = x$;
%		\item $V_i = \sigma(W_{i} \cdot V_{i-1} + B_i)$  for $i \in [k]$, where 
%		$\sigma$ is an activation function applied element-wise, $B_i$ is the vector of biases assigned to layer $N_i$ and $W_i$ is the matrix of weights assigned to edges between nodes in subsequent layers $N_{i-1}, N_{i}$, for $i \in [h+1]$.. 
%	\end{itemize}
%	Then, $\model{\Theta}(x) = V_{h+1} = W_{h+1} \cdot V_h + B_{h+1}$.
%	\label{def:ffnn_computation}
%\end{definition}

The outer minimisation problem is equivalent to the following MILP program, {where} the superscripts $j$ on weight matrices and bias vectors indicate they are model parameters of the $j$-th model $\model{\Theta}^j \in \Delta'$:
\begin{subequations}\label{OP:outerminimisation}
	\begin{align}
		\label{obj:OP MILP}
		\underset{x', \gamma, \lambda}{\min}  & \quad \| x - x' \|_1 \\
		\text{s.t.} \quad
		\label{const1:OP MILP}
		& V_0^j = x',\\
		\label{const2:OP MILP}
		& \gamma_i^j \in \{0,1\}^{|N_i|}, ~ i \in [h], j \in [|\Delta'|] \\
		\label{const3:OP MILP}		
		& 0 \leq V_i^j \leq M \gamma_i^j, ~ i \in [h]  , j \in [|\Delta'|] \\
		%\label{const4:OP MILP}		
		%& V_i \leq M \gamma_i, ~ i \in [h]\\		
		%\label{const5:OP MILP}		
		%& V_i \geq W_i V_{i-1} + B_i, ~ i \in [h]\\	
		\label{const6:OP MILP}		
		& W_i^j V_{i-1}^j + B_i^j \leq V_i^j \leq (W_i^j V_{i-1}^j + B_i^j) + M (1-\gamma_i^j), \\& \nonumber i \in [h], j \in [|\Delta'|] \\
		\label{const7:OP MILP}
		& W_{h+1}^j V_h^j + B_{h+1}^j \geq 0, j \in [|\Delta'|] \\
            %\label{const8:OP MILP}
            %& \delete{\lambda_l \in [0, 1], \sum_{l}^{|S_{k, \Delta} \cup \{x\}|}=1, l \in [|S_{k, \Delta} \cup \{x\}|]}\\
            & \lambda_l \in [0, 1],~ l \in [|S_{k, \Delta} \cup \{x\}|], ~\sum_{l=1}^{|S_{k, \Delta} \cup \{x\}|} \lambda_l=1 \nonumber \\
            \label{const9:OP MILP}
            & x' = \sum_{l=1}^{|S_{k, \Delta} \cup \{x\}|}\lambda_l x'_l, x'_l\in S_{k, \Delta} \cup \{x\}
	\end{align}
\end{subequations}

Constraints~\eqref{const1:OP MILP} - \eqref{const7:OP MILP} and constraint \eqref{const9:OP MILP} correspond respectively to the robustness and plausibility requirement {in} \eqref{eqn:outerrob} - \eqref{eqn:outerplaus}.

The inner maximisation program can be formulated %\delete{to}\JL
{as} the following MILP program, {where} the superscripts $0$ on weight matrices and biases indicate they are model parameters of the original model $\model{\Theta}$, and $\delta$ is the bound of model magnitude change specified in $\Delta$:
\begin{subequations}\label{OP:innermaximisation}
	\begin{align}
		\label{obj:OP2 MILP}
		\underset{V, W, B, \gamma}{\max}  & \quad -V_{h+1} \\
		\text{s.t.} \quad
		\label{const1:OP2 MILP}
		& V_0 = x',\\
		\label{const2:OP2 MILP}
		& \gamma_i \in \{0,1\}^{|N_i|}, ~ i \in [h]\\
		\label{const3:OP2 MILP}		
		& 0 \leq V_i \leq M \gamma_i, ~ i \in [h]  \\
		%\label{const4:OP2 MILP}		
		%& V_i \leq M \gamma_i, ~ i \in [h]\\		
		%\label{const5:OP2 MILP}		
		%& V_i \geq W_i V_{i-1} + B_i, ~ i \in [h]\\	
		\label{const6:OP2 MILP}		
		& W_i V_{i-1} + B_i \leq V_i \leq (W_i V_{i-1} + B_i) + M (1-\gamma_i), \\&\nonumber%\qquad\qquad\qquad\qquad\qquad\qquad 
  i \in [h]\\
		\label{const7:OP2 MILP}
		& V_{h+1} = W_{h+1} V_h + B_{h+1} \\
            \label{const8:OP2 MILP}
            & W_i^0 - \delta \leq W_i \leq W_i^0 + \delta, ~ i \in [h+1] \\
            \label{const9:OP2 MILP}
            & B_i^0 - \delta \leq B_i \leq B_i^0 + \delta, ~ i \in [h+1]
	\end{align}
\end{subequations}

Due to the flexibility of such MILP programs, the framework accommodates continuous, ordinal, and categorical features \cite{mohammadi2021scaling}. Specific requirements like feature immutability or associations between features can also be encoded \cite{ustun2019actionable}. These MILP problems can be directly solved using off-the-shelf solvers such as Gurobi \cite{gurobi}. %\JL{JL: Add reference for Gurobi.}

\subsection{Soundness, Completeness and Convergence of Algorithm \ref{alg:algo1}}
\label{subsec:theoreticalstuff}
We now discuss the soundness and completeness of our method by restricting the search space for the CE to the plausible region $\mathcal{X}_{plaus}$. From its definition, the vertices (except the input $x$) of $\mathcal{X}_{plaus}$ are $\Delta$-robust, which {thus} satisfies the robustness requirement of our target problem (Definition~\ref{def:robandplausproblemminmax}). %\delete{already}
This means that there exist at least $k$ points in the search space satisfying constraint~\eqref{eqn:plausobjminmax} that also satisfy constraint~\eqref{eqn:robobjminmax}, making these points feasible solutions for the target problem. We {may thus make} 
%\delete{conclude with}
the following remark:

\begin{proposition}
    Algorithm 1 is sound and complete if $\exists$ $x'\in \dataset$ such that $x'$ is $\Delta$-robust.
\label{def:soundnesscompleteness}
\end{proposition}

Next, we adapt the method in \cite[Section 5.2]{mutapcic2009cutting} to provide an upper bound on the maximum number of iterations of Algorithm~\ref{alg:algo1}.

\begin{proposition}
    Given the requirements of Algorithm~\ref{alg:algo1}, assume the classifier $\model{\Theta}$ is Lipschitz continuous in $x'$. Then{,} the maximum number of iterations before Algorithm~\ref{alg:algo1} terminates is bounded. %Algorithm~\ref{alg:algo1} will terminate in a finite number of steps. %$(CL+1)^d$, where $C$ is a constant.
\end{proposition}

\begin{proof} 
%\todo{The following looks basically the same as in the original 2009 paper, is this fine? \\
%Comment: I think it is fine to adapt their proof by giving the credit to the reference. }

Firstly, we assume two small tolerance variables $\sigma > t > 0$ and modify the robustness constraint \eqref{eqn:robobjminmax} of Definition~\ref{def:robandplausproblemminmax} to: $\underset{\model{\Theta'} \in \Delta}{\max} [-\model{\Theta'}(x') + \sigma] \leq t$, such that the correctness of the robustness guarantee is not affected. The termination condition for Algorithm~\ref{alg:algo1} therefore becomes $-\model{\Theta'}(x') + \sigma \leq t$.

Consider the plausible CE problem (Definition~\ref{def:plausiblece} with the validity constraint modified to $-\model{\Theta}(x') + \sigma \leq t)$, which is the problem solved by the first execution ({iteration} $1$) of the outer minimisation problem in Algorithm~\ref{alg:algo1}. We denote its feasible region as $\mathcal{F}$. Suppose $\model{\Theta}$ is a ReLU NN without the final (softmax or sigmoid) activation layer, then $\model{\Theta}$ is Lipschitz continuous. 
Let $f(x', \model{\Theta}):= -\model{\Theta}(x') + \sigma$, then $f$ is Lipschitz continuous in $x'$ over $\mathcal{F}$ with {some} Lipschitz constant $L$. For a distance metric $dist: \mathbb{R}^d \times \mathbb{R}^d \rightarrow \mathbb{R}^+$, and {any} $x_1, x_2 \in \mathcal{F}$, we have: 

\begin{equation}
    |f(x_1, \model{\Theta}) - f(x_2, \model{\Theta})| \leq L \times dist(x_1, x_2)
\label{eqn:lipcon}
\end{equation}

At %\delete{step}\JL
{iteration} $m$, we denote the CE found by the outer minimisation as $x'^{(m)}$, and the shifted model found by the inner maximisation as $\model{\Theta}^{(m)}$. Then, $f(x'^{(m)}, \model{\Theta}^{(m)}):= -\model{\Theta}^{(m)}(x'^{(m)}) + \sigma$. Assume at step $m$ the algorithm has not terminated, then 
\begin{equation}
    f(x'^{(m)}, \model{\Theta}^{(m)}) > t
\label{eqn:fm}
\end{equation}

For %\delete{any}\JL
{the} iteration steps %\delete{$n$ later than step $m$}\JL
{$n>m$}, $x'^{(n)}$ is required to be valid on %\delete{$\model{\Theta^{(m)}}$}
$\model{\Theta}^{(m)}$ as specified in the outer minimisation problem, we therefore have:
\begin{equation}
    f(x'^{(n)}, \model{\Theta}^{(m)})\leq 0
\label{eqn:fn}
\end{equation}

Combining %\delete{the two equations} 
\eqref{eqn:fm} and \eqref{eqn:fn} %\delete{, for step $m<n$,}\JL
{yields}
\begin{align}
    %& \delete{f(x'^{(m)}, \model{\Theta}^{(m)}) - f(x'_n, \model{\Theta}^{(m)}) > t.} \nonumber\\
    f(x'^{(m)}, \model{\Theta}^{(m)}) - f(x'^{(n)}, \model{\Theta}^{(m)}) > t
\label{eqn:fmn}
\end{align}

Further combining \eqref{eqn:fmn} with \eqref{eqn:lipcon}, for the %\delete{step $m<n$}\JL
{iteration steps $n > m$},
\begin{equation}
    dist(x'^{(m)}, x'^{(n)}) > \frac{t}{L}
\label{eqn:dist}
\end{equation}

Consider the balls %\delete{$\{B_i, \text{ } i=1, \ldots, m\}$}\JL
{$B_i, i=1,\ldots,m,$} of diameter $\frac{t}{L}$ centred at each intermediate result of the outer minimisation problem, $x'^{(i)}$. From \eqref{eqn:dist}, it can be concluded that for any two intermediate $x'^{(i)}$, $x'^{(j)}$, %\delete{$1<i<m$ and $1<j<m$}\JL
{$1 < i,j < m$}, $dist(x'^{(i)},x'^{(j)}) > \frac{t}{L}$, and ${x'^{(i)}}$ and ${x'^{(j)}}$ are the centres of the balls $B^{(i)}$ and $B^{(j)}$. Therefore, any two circles will not intercept. The total volume of these balls is thus $m\times U \times (\frac{t}{L})^d$, where $U$ is the unit volume in $\mathbb{R}^d$.

Consider a ball that encompasses the feasible solution region $\mathcal{F}$ %\delete{,} \JL
{and} let $R$ be its radius. We know that %\delete{$\{x'_i, \text{ } i=1, \ldots, m\}$}\JL
{$x'_i, i=1,\ldots,m,$} are all within the feasible region $\mathcal{F}$, therefore, the ball $B$ that has a radius $R+\frac{t}{2L}$ will cover the spaces of the small balls %\delete{$\{B_i, \text{ } i=1, \ldots, m\}$}\JL
{$B_i, i=1,\ldots,m$}. Also, the volume of $B$ %\delete{,}\JL
{is} $U\times (2R + \frac{t}{L})^d$ %\delete{,}\JL
{and} will be greater than the total volume of the small balls{, which means}:
\begin{align*}
    %&\delete{U\times (2R + \frac{t}{L})^d > m\times U \times (\frac{t}{L})^d \quad \implies \quad m < (\frac{2RL}{t} + 1)^d} \nonumber\\
    U\times \left(2R + \frac{t}{L} \right)^d > m\times U \times \left( \frac{t}{L} \right)^d \quad \implies \quad m < \left(\frac{2RL}{t} + 1 \right)^d
%\label{eqn:iterationnumberupperbound}
\end{align*}

It can be concluded that the step number at which Algorithm~\ref{alg:algo1} has not terminated is bounded above by the $(\frac{2RL}{t} + 1)^d$.
\end{proof}

\section{Experiments}
\label{sec:experiments}
In this section, we demonstrate that our proposed method achieves state-of-the-art performances compared with existing robust CEs generation methods. 

\paragraph{Datasets and Classifiers.} 
Our experiments use four benchmark datasets in financial and legal contexts: the Adult Income (ADULT), COMPAS, Give Me Some Credits (GMC), and HELOC datasets. We adopt the pre-processed versions available in the CARLA library \cite{pawelczyk2021carla} where each dataset contains binarised categorical features and min-max scaled continuous features. Labels 0 and 1 are the unwanted and the desired class, respectively. We split each dataset into two halves. We use the first half for training NNs with which the robust CEs are generated, and the second half for model retraining and evaluating the robustness of the CEs.

For making predictions and generating CEs, the NNs contain two hidden layers with ReLU activation functions. They are trained using the Adam optimiser with a batch size of 32, and under the standard $80\%, 20\%$ train-test dataset split setting. The classifiers achieved 84\%, 85\%, {94}{\%}, and 76\% accuracies on the test set of ADULT, COMPAS, {GMC}, and HELOC datasets{, respectively}.

The retrained models have the same hyperparameters and training procedures as the original classifiers. Following the experimental setup in previous works ~\cite{pmlr-v162-dutta22a,robovertime,nguyen2022robust,blackconsistent,upadhyay2021towards}, for each dataset, we train 10 new models using both halves of the dataset to simulate the possible retrained models after new data are collected. We also train 10 new models using 99\% of the first half of the dataset (different 1\% data are discarded for each training), to simulate the leave-one-out retraining procedures. The random seed is perturbed for retraining. These 20 retrained models are used for evaluating the robustness of CEs.

\paragraph{Evaluation Metrics.}
The CEs are evaluated by the following metrics for their proximity, plausibility, and robustness. 
\begin{itemize}
    \item {$\ell_1$} measures the average $L_1$ distance between a CE and its corresponding input. 
    \item $lof$ is the average 10-Local Outlier Factor \cite{lof} of the generated CEs, which indicates to what extent a data point is an outlier wrt its $k$ nearest neighbours in a specified dataset. $lof$ values close to $1$ indicate inliers, larger values (especially if greater than $1.5$) indicate outliers.
    \item $vr$, the validity of CEs on the retrained models, is defined as the average percentage of CEs that remain valid (classified to class 1) under the retrained models.
    \item $v\Delta$ {is} the percentage of CEs that are $\Delta$-robust. The bound {of} model parameter changes $\delta$ is specified to be the same as the value used in our algorithm.
\end{itemize}

\paragraph{Baselines.} We compare our method with six %\JL{*JL: You say five in Abstract} 
state-of-the-art methods for generating CEs, including five {which target} robust{ness}. WCE %\todo{AR: should we call it WCE or something to differentiate from CEs?} 
\cite{Wachter_17} is the first method to generate CEs for NNs, which minimises the {$\ell_1$} distance between the CEs and the inputs. Robust Bayesian Recourse (RBR) \cite{nguyen2022robust} addresses the proximity, robustness, and plausibility of CEs. RobXNN \cite{pmlr-v162-dutta22a} is a nearest-neighbour-based method that focuses on a different {notion of }robustness {to model changes}.  Robust Algorithmic Recourse (ROAR) \cite{upadhyay2021towards} optimises for proximity and the same $\Delta$ notion of robustness. Proto-R and MILP-R are the methods proposed by \cite{oursaaai23} which embed the $\Delta$-robustness test into the base methods {of} \cite{van2021interpretable} and \cite{mohammadi2021scaling}. %The bounds for model parameter perturbation in $\Delta$ for Proto-R, and MILP-R are set to the same value as in our method. 
For all methods including ours, we tune their hyperparameters to maximise the validity after retraining $vr$.
%Wachter, rbr, robxnn, roar, proto-r, milp-r, 

\begin{table*}[ht!]
    \centering
    %\resizebox{\columnwidth}{!}{
    \begin{tabular}{ccccc|cccc|cccc|cccc}
         \cline{2-17}
         &
        \multicolumn{4}{c}{ADULT} &
        \multicolumn{4}{c}{COMPAS} &
        \multicolumn{4}{c}{GMC} &
        \multicolumn{4}{c}{HELOC} \\
        %{\textbf{method}} 
        & 
        \!\!\textbf{vr}$\uparrow$\!\! & 
        \!\!\textbf{v$\Delta$$\uparrow$}\!\! &
        \!\!$\ell_1$$\downarrow$\!\! &
        \!\!\textbf{lof$\downarrow$}\!\! &
        \!\!\textbf{vr}$\uparrow$\!\! & 
        \!\!\textbf{v$\Delta$$\uparrow$}\!\! &
        \!\!\!$\ell_1$$\downarrow$\!\!\! &
        \!\!\textbf{lof$\downarrow$}\!\! &
        \!\!\textbf{vr}$\uparrow$\!\! & 
        \!\!\textbf{v$\Delta$$\uparrow$}\!\! &
        \!\!\!$\ell_1$$\downarrow$\!\!\! &
        \!\!\textbf{lof$\downarrow$}\!\! &
        \!\!\textbf{vr}$\uparrow$\!\! & 
        \!\!\textbf{v$\Delta$$\uparrow$}\!\! &
        \!\!\!$\ell_1$$\downarrow$\!\!\! &
        \!\!\textbf{lof$\downarrow$}\!\! \\
        \hline
        \!\!\!\!WCE\!\!\!\! & 
        \!\!89\!\! &
        \!\!78\!\! &
        \!\!.175\!\! &
        \!\!1.59\!\! & 
        \!\!57\!\! &
        \!\!0\!\! &
        \!\!.170\!\! &
        \!\!1.81\!\! &
        \!\!84\!\! &
        \!\!18\!\! &
        \!\!.148\!\! &
        \!\!2.80\!\! &         
        \!\!49\!\! &
        \!\!0\!\! &
        \!\!.045\!\! &
        \!\!1.16\!\! \\
        \!\!\!\!RBR\!\!\!\! & 
        \!\!100\!\! &
        \!\!0\!\! &
        \!\!.031\!\! &
        \!\!1.28\!\! & 
        \!\!100\!\! &
        \!\!62\!\! &
        \!\!.043\!\! &
        \!\!1.34\!\! & 
        \!\!90\!\! &
        \!\!0\!\! &
        \!\!.050\!\! &
        \!\!1.31\!\! &         
        \!\!80\!\! &
        \!\!0\!\! &
        \!\!.038\!\! &
        \!\!1.10\!\! \\
        \!\!\!\!RobXNN\!\!\!\! & 
        \!\!100\!\! &
        \!\!82\!\! &
        \!\!.064\!\! &
        \!\!1.28\!\! & 
        \!\!100\!\! &
        \!\!82\!\! &
        \!\!.050\!\! &
        \!\!1.11\!\! &
        \!\!100\!\! &
        \!\!96\!\! &
        \!\!.073\!\! &
        \!\!1.35\!\! &
        \!\!100\!\! &
        \!\!30\!\! &
        \!\!.073\!\! &
        \!\!1.04\!\! \\    
        \!\!\!\!ROAR\!\!\!\! & 
        \!\!99\!\! &
        \!\!98\!\! &
        \!\!.279\!\! &
        \!\!1.96\!\! & 
        \!\!98\!\! &
        \!\!100\!\! &
        \!\!.219\!\! &
        \!\!2.84\!\! &
        \!\!96\!\! &
        \!\!88\!\! &
        \!\!.188\!\! &
        \!\!4.22\!\! &         
        \!\!98\!\! &
        \!\!98\!\! &
        \!\!.109\!\! &
        \!\!1.57\!\! \\
        \!\!\!\!PROTO-R\!\!\!\! & 
        \!\!98\!\! &
        \!\!55\!\! &
        \!\!.068\!\! &
        \!\!1.60\!\! & 
        \!\!100\!\! &
        \!\!100\!\! &
        \!\!.084\!\! &
        \!\!1.36\!\! &
        \!\!100\!\! &
        \!\!100\!\! &
        \!\!.066\!\! &
        \!\!1.49\!\! &  
        \!\!100\!\! &
        \!\!100\!\! &
        \!\!.057\!\! &
        \!\!1.21\!\! \\
        \!\!\!\!{MILP-R}\!\!\!\! & 
        \!\!100\!\! &
        \!\!100\!\! &
        \!\!.024\!\! &
        \!\!1.69\!\! & 
        \!\!100\!\! &
        \!\!100\!\! &
        \!\!.040\!\! &
        \!\!1.71\!\! &
        \!\!100\!\! &
        \!\!100\!\! &
        \!\!.059\!\! &
        \!\!2.08\!\! &         
        \!\!100\!\! &
        \!\!100\!\! &
        \!\!.044\!\! &
        \!\!2.48\!\! \\
        \!\!\!\!\textbf{OURS}\!\!\!\! & 
        \!\!100\!\! &
        \!\!100\!\! &
        \!\!.046\!\! &
        \!\!1.22\!\! & 
        \!\!100\!\! &
        \!\!100\!\! &
        \!\!.039\!\! &
        \!\!1.24\!\! &
        \!\!100\!\! &
        \!\!100\!\! &
        \!\!.058\!\! &
        \!\!1.24\!\! &         
        \!\!100\!\! &
        \!\!100\!\! &
        \!\!.057\!\! &
        \!\!1.04\!\! \\
        \hline
    \end{tabular}
   % }
    \caption{Evaluations of PROPLACE (OURS) and baselines on NNs. The $\uparrow$ ($\downarrow$) indicates that higher (lower) values are preferred for the evaluation metric.}%\todo{make best values bold or underline?} \todo{2 decimal points for L1}}
    \label{tab:results}
\end{table*}

\paragraph{Results.} We randomly select 50 test points from each dataset that are classified to {be} the unwanted class, then apply our method and each baseline to generate CEs for these test points. Results are {shown} in Table~\ref{tab:results}. %\todo{Also report computation time in an appendix?}

As a non-robust baseline, the WCE method is the least robust while producing high {$\ell_1$} costs and poor plausibility. Though RBR shows the lowest $\ell_1$ results on three datasets, it has only moderate robustness against the naturally retrained models and is not $\Delta$-robust on any dataset. The rest of the baselines all show strong robustness on at least three datasets, with our method having slightly better $vr$ and $v\Delta$ results, evaluated at 100\% in every experiment. This indicates that {our method} PROPLACE can not only guarantee robustness under bounded model parameter changes but also induce reliable robustness against unbounded model changes. In terms of plausibility, our method shows the best lof score in most experiments. Therefore, our method has addressed the %\delete{limitations}\JL
{limitation} in the literature that no method optimises for guaranteed $\Delta$-robustness and plausibility. Though the two properties have established trade-offs with proximity \cite{pawelczyk2020counterfactual,pawelczyk2022probabilistically,upadhyay2021towards}, our method still shows {$\ell_1$} costs lower than all methods except RBR, which is significantly less robust, and MILP-R, which finds outliers. For {the} COMPAS dataset, our method has the best proximity result among all baselines. 

Note that the PROTO-R baseline from the work which proposed certification for $\Delta$-robustness failed to find $\Delta$-robust CEs on the ADULT dataset, as was the case in their results (see Table 1, \cite{oursaaai23}). This is due to the fact that their method rely heavily on a base method to find CEs, and it is not straightforward to be always able to direct the hyperparameters search for optimising $\Delta$-robustness. With improved soundness and completeness ({Proposition}~\ref{def:soundnesscompleteness}), PROPLACE always finds provably robust results.

\section{Conclusions}
\label{sec:conclusion}
We proposed a robust optimisation framework PROPLACE to generate provably robust and plausible CEs for neural networks. The method addresses the limitation in the literature that existing methods lack formal robustness guarantees to bounded model parameter changes and do not generate plausible CEs. We proved the soundness, completeness, and convergence of PROPLACE. Through a comparative study, we show the efficacy of our method, demonstrating the best robustness and plausibility results with better proximity than the most robust baselines. Despite the specific form of robustness we target, PROPLACE is also empirically robust to model retraining with unbounded parameter changes. 
Future work could include investigating the properties of actionability and diversity, evaluations with user studies, and investigating connections between $\Delta$-robustness and different notions of robustness measures.

\section*{Acknowledgement}
{Jiang, Rago and Toni were partially funded by J.P. Morgan and by the Royal Academy of Engineering under the Research Chairs and Senior Research Fellowships scheme. 
Jianglin Lan is supported by a Leverhulme Trust Early Career Fellowship under Award ECF-2021-517. 
Leofante is supported by an Imperial College Research Fellowship grant. 
Rago and Toni were partially funded by the European Research Council (ERC) under the European Union’s Horizon 2020 research and innovation programme (grant agreement No. 101020934). 
Any views or opinions expressed herein are solely those of the authors listed.}

\bibliographystyle{named}
\bibliography{bib}

\end{document}